\documentclass[a4paper]{easychair}

\usepackage{doc}
\usepackage[export]{adjustbox}

\usepackage[misc]{ifsym}
\usepackage{makecell}
\usepackage{longtable}
\usepackage{tabularx}
\usepackage{array}
\usepackage{wrapfig}
\usepackage{paralist}
\usepackage{xspace}
\usepackage{amsmath}
\usepackage{amssymb}
\usepackage{url}
\usepackage{algorithm}
\usepackage{algorithmic}
\usepackage{booktabs}
\usepackage{siunitx}

\usepackage{cite}
\usepackage{hyperref}
\hypersetup{colorlinks,breaklinks,urlcolor=blue,linkcolor=blue,citecolor=blue}

\newtheorem{definition}{Definition}
\newtheorem{lemma}{Lemma}

\newcommand{\cegarnn}{\texttt{CEGAR-NN}}
\newcommand{\cegartwo}{\texttt{CEGARETTE}}
\newcommand{\cegarnntwo}{\texttt{CEGARETTE-NN}}

\newcommand{\sat}{\texttt{SAT}}
\newcommand{\unsat}{\texttt{UNSAT}}

\newcommand{\SBT}{\texttt{SBT}}
\newcommand{\IBP}{\texttt{IBP}}
\newcommand{\DeepPoly}{\texttt{DeepPoly}}

\newcommand{\abstractOp}{\texttt{abstract}}
\newcommand{\refineOp}{\texttt{refine}}
\newcommand{\queryAbstractOp}{\texttt{queryAbstract}}
\newcommand{\queryRefineOp}{\texttt{queryRefine}}

\newcommand{\relu}{\text{ReLU}\xspace{}}

\usepackage{tikz,ifthen,pgfplots}

\newcommand{\mysubsection}[1]{\medskip\noindent\textbf{#1}}

\usetikzlibrary{arrows,trees,backgrounds,automata,shapes,decorations,plotmarks,fit,calc,positioning,shadows,chains}
\tikzstyle{every pin edge}=[<-,shorten <=1pt]
\tikzstyle{neuron}=[circle,fill=black!25,minimum size=17pt,inner sep=0pt]
\tikzstyle{input neuron}=[neuron, fill=orange!50]
\tikzstyle{output neuron}=[neuron, fill=purple!50]
\tikzstyle{hidden neuron}=[neuron, fill=blue!50]
\tikzstyle{annot} = [text width=3.5em, text centered]
\tikzstyle{nnedge} = [-{stealth},shorten >=0.1cm, shorten <=0.05cm,line width=0.5pt,black]

\usetikzlibrary{angles,quotes}

\title{Tighter Abstract Queries in\\ Neural Network Verification%
}

\author{
	Elazar Cohen\inst{1}\thanks{Equal Contribution}
\and
	Yizhak Yisrael Elboher\inst{1}$^\star$(\Letter)
\and
	Clark Barrett\inst{2}
\and
	Guy Katz\inst{1}
}

\institute{
	The Hebrew University of Jerusalem, Jerusalem, Israel \\
	\{elazar.cohen1, yizhak.elboher, g.katz\}@mail.huji.ac.il
	\and
	Stanford University, Stanford, USA \\
	barrett@cs.stanford.edu 
}

\authorrunning{Cohen, Elboher, Barrett and Katz}

\titlerunning{Tighter Abstract Queries in Neural Networks Verification}

\begin{document}

\maketitle

\begin{abstract}
  Neural networks have become critical components of reactive
  systems in various domains within computer science. Despite their
  excellent performance, using neural networks entails numerous risks
  that stem from our lack of ability to understand and reason about
  their behavior. Due to these risks, various formal methods have been
  proposed for verifying neural networks; but unfortunately, these typically
  struggle with scalability barriers. Recent attempts have demonstrated that
  abstraction-refinement approaches could play a significant
  role in mitigating these limitations; but these approaches can often
  produce networks that are so abstract, that they become unsuitable
  for verification.  To deal with this issue, we
  present \cegartwo{}, a novel verification mechanism where both the
  system and the property are abstracted and refined simultaneously.
  We observe that this approach allows us to produce abstract networks
  which are both small and sufficiently accurate, allowing for quick
  verification times while avoiding a large number of refinement
  steps.  For evaluation purposes, we implemented \cegartwo{} as an
  extension to the recently proposed \cegarnn{} framework. Our results
  are highly promising, and demonstrate a significant improvement in
  performance over multiple benchmarks.
\end{abstract}

\providecommand{\keywords}[1]
{
	\textbf{\textit{Keywords: }} #1
}


\setcounter{tocdepth}{2}

%
%

\section{Introduction}
Deep neural networks (DNNs) have become state-of-the-art technology in
many fields~\cite{Liu2017}, including image processing~\cite{heZaReSu2015deep},
computational photography~\cite{BaOzSi2019}, speech
recognition~\cite{GuQiChPaZhYuHaWaZhWuPa2020, AlSaViGlIf2020}, natural
language processing~\cite{DeChLeTo2019}, video
processing~\cite{KaToShLeSuFe2014,ApAdMeMePa2021}, autonomous
driving~\cite{BoDeDwFiFlGoJaMoMuZhZhZhZi16}, and many
others. Nowadays, they are even increasingly being used as critical
components in various systems~\cite{ShFaDeChAl2021, SuMiLiJi2021,
	LeMcCo2021, MaAbNeCh2017}, and society's reliance on them is
constantly increasing.

Despite these remarkable achievements, neural networks suffer from
multiple limitations, which undermine their reliability: the training
process of DNNs is based on assumptions regarding the data, which may
fail to hold later on~\cite{JoKh2019,GuLi2020}; the training process
might cause over-fitting of the DNN to specific
data~\cite{SrHiKrSuSa2014,ZhViMuBe2018}; and, independently of the
above, there are various attacks that can fool a DNN into
performing unwanted actions~\cite{ReZhQiLi2020,HaYaHaDeHuJiAn2019}.

In order to overcome these difficulties and ensure the correctness and
safety of DNNs, the formal methods community has been devising
techniques for verifying them~\cite{KaBaDiJuKo17Reluplex,
	KaHuIbJuLaLiShThWuZeDiKoBa19Marabou, Ehlers2017,
	MuMaSiPuVe2022, WaZhXuLiJaHsKo2021}. Techniques for neural network
verification (NNV) receive as input a neural network and a set of
constraints over its input and output, and check whether there exists
an input/output pair that satisfies these constraints. Typically, the
constraints encode the negation of some desirable property, and so
such a pair constitutes a counter-example (the \sat{} case); whereas
if no such pair exists, the desired property holds (the \unsat{} case).
NNV has been studied extensively in recent years, and many different
verifiers have been proposed~\cite{KaBaDiJuKo17Reluplex,DuChJhSaTi2019,PaWuGrCaPaBa2021,DuJhSaTi18,
	KaHuIbJuLaLiShThWuZeDiKoBa19Marabou, Ehlers2017, 
	MuMaSiPuVe2022, WaZhXuLiJaHsKo2021}.   However,
scalability remains a fundamental barrier, both theoretical and
practical, which limits the use of NNV engines: generally, increasing the
number of neurons of the verified neural network implies an
exponential increase in the complexity of the verification
problem~\cite{KaBaDiJuKo17Reluplex, IsZoBa23}.

In order to alleviate this scalability limitation, recently there have
been attempts to apply \emph{abstraction techniques} within
NNV~\cite{PrAf2020, SoTh2020, AsHaKrMu20, ElGoKa2020, ElCoKa2022,
	LiXiShSoXuMi2022}, often focusing on the counter-example guided
abstraction refinement (CEGAR) framework~\cite{ClGrJhLuVe10CEGAR}.
CEGAR is a well-known approach, aimed at expediting the solving of
complex verification queries, by \emph{abstracting} the model being
verified into a simpler one --- in such a way that if the simpler
model is safe, i.e.~the query is \unsat{}, then so is the original
model.  In case of a \sat{} answer from the verifier, we check whether
the satisfying assignment of the abstract model is also a satisfying
assignment of the original. If so, the original query is declared
\sat{}, the satisfying assignment is returned, and the process
ends. Otherwise, the satisfying assignment is called \emph{spurious}
(or \emph{erroneous}), indicating that the abstract query is too
coarse, and should be refined into a slightly ``less abstract'' query,
which is a bit closer to the original model (but is hopefully still
smaller). CEGAR has been successfully used in various formal
method
applications~\cite{AnLiSa2005,BeLe2016,JaKlMaCl2016,HeRuGo2017,JoErBe2015},
and also in the context of NNV~\cite{ElGoKa2020, ElCoKa2022,
	LiXiShSoXuMi2022}. Typically, these approaches generate an
\emph{abstract neural network}, which is smaller than the original,
and is created by the merging of neurons of the original network. The
refinement, accordingly, is performed by splitting previously merged
neurons. Other related approaches have also considered abstracting the
ranges of values obtained by neurons in the network~\cite{PrAf2020,
	SoTh2020, AsHaKrMu20, PuTa10}.

The general motivation for abstraction schemes is that smaller,
abstract networks are easier to verify. While this is often true,
smaller networks tend to be far less precise, and verifying them often
requires multiple refinement steps~\cite{ElGoKa2020, PuTa10}. In
extreme cases, these multiple refinement steps can render the
verification process slower than directly verifying the original
network~\cite{ElGoKa2020}. Here, we seek to tackle this problem, by
improving the abstract verification queries. We propose a novel
verification scheme for DNNs, wherein abstraction and refinement
operations include altering not only the network (as
in~\cite{ClGrJhLuVe10CEGAR, ElGoKa2020, ElCoKa2022, LiXiShSoXuMi2022,
	PuTa10}), but also the property being verified. The motivation is to
render the abstract properties more \emph{restrictive}, in a way that
will reduce the number of spurious counter-examples encountered during
the verification process; but at the same time, ensure that the
abstract queries still maintain the over-approximation property: if
the abstract query is \unsat{}, the original query is \unsat{} too.
The key idea on which our approach is based is to compute a
\emph{minimal difference} between the outputs of the abstract network
and the original network, and then use this minimal difference to
tighten the property being verified, in a sound manner. 

Our approach is not coupled to any specific DNN verification method,
and can use multiple DNN verifiers as black-box backends.  For
evaluation purposes, we implemented it on top of the Marabou DNN
verifier~\cite{KaHuIbJuLaLiShThWuZeDiKoBa19Marabou}. We then tested
our approach on the ACAS-Xu benchmarks~\cite{JuLoBrOwKo2016} for
airborne collision avoidance, and also on MNIST benchmarks for digit
recognition~\cite{Le98}. Our results indicate that property
abstraction affords a significant increase in the number of queries
that can be verified within a given timeout, as well as a sharp
decrease in the number of refinement steps performed.

To summarize, our contributions are as follows:
\begin{inparaenum}[(i)]
	\item we present \cegartwo{}, a novel CEGAR
	framework for DNN verification, that abstracts not only the network
	but also the property being verified;
	\item we provide a publicly available implementation of our approach, \cegarnntwo{}; and
	\item we use our implementation to demonstrate the practical
	  usefulness of our approach.
\end{inparaenum}

The rest of this paper is organized as follows. In
Section~\ref{sec:background}, we provide a brief background on neural
networks and their verification, followed by an explanation of the CEGAR
framework and its implementation for neural network verification. In
Section~\ref{sec:cegartwo}, we describe our novel verification
framework \cegartwo{}. In Section~\ref{sec:cegarnntwo}, we discuss how
to apply this framework for abstracting and refining DNNs, followed by
an evaluation in Section~\ref{sec:evaluation}.  Related work is
discussed in Section~\ref{sec:relatedWork}, and we conclude in
Section~\ref{sec:conclusion}.

\section{Background}
\label{sec:background}

\subsection{Neural Networks}

A neural network~\cite{Goodfellow-et-al-2016} is a directed graph,
comprised of a sequence of \textit{layers}: an input layer, followed
by one or more consecutive hidden layers, and finally an output
layer. A layer is a collection of nodes, also referred to as
\textit{neurons}.  Here we focus on \textit{feed-forward} neural
networks, where the values of neurons are computed based on values of
neurons in preceding layers. Thus, when the network is evaluated,
values are assigned to neurons in its input layer; and they are then
propagated, layer after layer, through to the output layer.

We use $n_{i,j}$ to denote the $j$'th neuron of layer $i$.  Typically,
the value of neuron $n_{i,j}$, denoted as $v_{i,j}$, is given by the following
formula:
\[
v_{i,j} = act_{i,j}(b_{i,j} + \sum_{k=1}^{l_{i-1}} w^i_{k,j}\cdot v_{i-1,k})
\]
where $l_{i-1}$ is the number of neurons in the $i-1$'th layer,
$act_{i,j}$ is a pre-defined (neuron-specific) activation function,
$w^i_{k,j}$ is the weight of the outgoing edge from $n_{i-1,k}$ to
$n_{i,j}$, 
$v_{i-1,k}$ is the value of the $k$'th neuron in the $i-1$'th layer, 
and $ b_{i,j} $ is the bias value of the $j$'th neuron in the
$i$'th layer. For simplicity, we assume here that the only
activation function in use is the Rectified Linear Unit (\relu{})
function~\cite{NaHi10}, which is defined by $\relu(x)=max(0,x)$, and
is very common in practice.

Fig.~\ref{fig:runningExample} depicts a small neural network. The
network has 3 layers, of sizes $l_1=1, l_2=2$ and $l_3=1$. Its weights
are $w^2_{1,1}=10$, $w^2_{1,2}=1$, $w^3_{1,1}=3$ and
$w^3_{2,1}=4$, and its biases are all zeros. 
For input $v_{1,1}=21$, node $n_{2,1}$ evaluates to 210
and node $n_{2,2}$ evaluates to 21 (both are positive, and hence not
changed by the \relu{} activation function). The output node $n_{3,1}$
then evaluates to $3\cdot 210 + 4\cdot 20 = 714$.

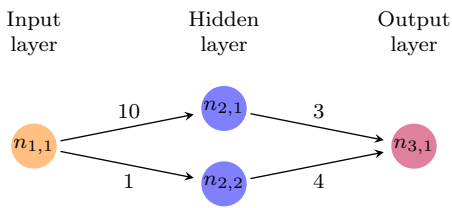
\begin{figure}[htp]
	\begin{center}
		\scalebox{1} {
			\def\layersep{2.5cm}
			\begin{tikzpicture}[shorten >=1pt,->,draw=black!50, node distance=\layersep,font=\footnotesize]
				
				\node[input neuron] (I-1) at (0,-1) {$n_{1,1}$};
				
				\path[yshift=0.5cm] node[hidden neuron] (H-1)
				at (\layersep,-1 cm) {$n_{2,1}$};
				\path[yshift=0.5cm] node[hidden neuron] (H-2)
				at (\layersep,-2 cm) {$n_{2,2}$};
				
				\node[output neuron] at (2*\layersep, -1) (O-1) {$n_{3,1}$};
				
				\draw[nnedge] (I-1) -- node[above] {$10$} (H-1);
				\draw[nnedge] (I-1) -- node[below] {$1$} (H-2);
				\draw[nnedge] (H-1) -- node[above] {$3$} (O-1);
				\draw[nnedge] (H-2) -- node[below] {$4$} (O-1);
				
				\node[annot,above of=H-1, node distance=1cm] (hl) {Hidden layer};
				\node[annot,left of=hl] {Input layer};
				\node[annot,right of=hl] {Output layer};
			\end{tikzpicture}
		}
		\caption{A simple, feed-forward neural network.}
		\label{fig:runningExample}
	\end{center}
\end{figure}

\subsection{Neural Network Verification}

The goal of neural network verification (NNV) is to determine the
satisfiability of a \emph{verification query}. A query is typically
defined to be a triple $\langle N,P,Q\rangle$, where:
\begin{inparaenum}[(i)]
	\item $N$ is a neural network;
	\item $P$ is an input property, which is a conjunction of constraints
	on the input neurons; and
	\item $Q$ is an output property, which is a conjunction of constraints
	on the output neurons~\cite{LiArLaBaKo19}.
\end{inparaenum}
The query is \sat{} if and only if there exists an input vector $x_0$
to $N$, such that $P(x_0)$ and $Q(N(x_0))$ both hold; in which case
the verifier returns $x_0$ as the counter-example. As we previously
mentioned, $Q$ typically represents some undesirable behavior of $N$
on inputs from $P$, and so the goal is to obtain an \unsat{} result.

Most existing verifiers focus primarily on \relu{} activation
functions, and we follow this line here. In addition, most existing
verifiers assume that $P$ is a conjunction of linear constraints on
the input values, and we again take the same path. Finally, we make
the simplifying assumption that $N$ has a single output neuron $y$,
and that the property $Q$ is of the form $y>c$. This assumption may
seem restrictive, but in fact, it does not incur any loss of
generality~\cite{ElGoKa2020}, and is sufficient for expressing many
properties of interest with arbitrary Boolean structure, via a simple
reduction.

In recent years, various methods have been proposed for solving the
verification problem (for a brief overview, see
Section~\ref{sec:relatedWork}). Our abstraction-refinement mechanism
is designed to be compatible with many of these techniques, as we
later describe.

\subsection{Counter-Example Guided Abstraction Refinement (CEGAR)}
Counter-example guided abstraction refinement
(CEGAR)~\cite{ClGrJhLuVe10CEGAR} is frequently used as part of
model-checking frameworks, and it has recently been applied to neural
network verification, as well.  The general framework, borrowed from
\cite{ElGoKa2020}, is presented in Algorithm~\ref{alg:abVerification}.
Given a \textit{verification query} $\langle N,P,Q\rangle$, we begin
by generating an \emph{abstract network} $N'$. Then, the CEGAR loop
starts, where in each iteration, we verify a query with the current
abstract network, $ \langle N',P,Q\rangle $. The abstract network $N'$
is constructed in a way that makes $ \langle N',P,Q\rangle $ an
over-approximation of $ \langle N,P,Q\rangle $: if the former is
\unsat{}, then so is the latter. Thus, if the underlying verifier
returns \unsat{} on the current query, we can stop and return
\unsat{}.  Otherwise, the verifier returns a satisfying assignment
$x_0$ for $ \langle N',P,Q\rangle $, and we check whether this $x_0$
constitutes a satisfying assignment for $ \langle N,P,Q\rangle $ as
well. If so, we return \sat{} and $x_0$ as the satisfying assignment,
and stop. Otherwise, we say that $x_0$ is a \emph{spurious}
counter-example, indicating that $N'$ is too coarse; in which case, we
\emph{refine} $N'$ into a new ``tighter'' network, $N''$, whose
verification is more likely to produce accurate results. This
refinement process is guided by the spurious counter-example
$x_0$. This general framework can be instantiated in many ways, depending on the
implementation of the \abstractOp{} and \refineOp{} operations.

\begin{algorithm}
	\caption{Abstraction-based DNN Verification($\langle N,P,Q\rangle$)}
	\begin{algorithmic}[1]
		\label{alg:cegar}
		\STATE $N'\leftarrow \abstractOp{}(N)$ \label{step:abstraction}
		\IF {\emph{Verify}($\langle N',P,Q\rangle$) is \unsat{}} \label{step:verify}
		\STATE return \unsat{}
		\ELSE
		\STATE Extract satisfying assignment $x_0$
		\IF {$x_0$ is a satisfying assignment for $N$}
		\STATE return \sat{}, $x_0$
		\ELSE
		\STATE $N''\leftarrow \refineOp{}(N',N,x_0)$ \label{step:refinement}
		\STATE $N'\leftarrow N''$
		\STATE Goto step~\ref{step:verify}
		\ENDIF
		\ENDIF
	\end{algorithmic}
	\label{alg:abVerification}
\end{algorithm}

There have been a few recent attempts to apply CEGAR in the
context of NNV~\cite{ElGoKa2020, LiXiShSoXuMi2022, ElCoKa2022}, all
following a similar approach.  At first, a preprocessing phase is
performed, and every hidden neuron in the network is classified
according to its effect on the network's output. Then, abstraction is
carried out by repeatedly merging pairs of neurons with the same type,
usually making the network significantly smaller than the original.

We focus here on one of these approaches, called
\cegarnn{}~\cite{ElGoKa2020}. There,  the
preprocessing phase initially splits each hidden neuron into 4
neurons, each belonging to one of 4 categories based on the effect of the
neuron on values of both the next layer's neurons and the output: the
output edges can be all positive (\textit{pos}) or all negative
(\textit{neg}), and the value of a neuron can increase the output when
being increased (\textit{inc}), or increase the output when being
decreased (\textit{dec}).  The splitting process changes the network's
architecture but does not change its output --- i.e., the preprocessed
network is completely equivalent to the original.  After splitting the
neurons and categorizing the new neurons, pairs of neurons from the
same layer that share a category can be merged into a single neuron.
The weights and biases of the new, merged neuron are determined by
aggregating the weights and biases of its constituent neurons, in a
way that depends on the category of these neurons, and which
guarantees that the abstract network's (single) output is always
greater than or equal to that of the original network when the two
networks are evaluated on the same input. This, combined with our
assumption that the output property is always of the form $y>c$,
guarantees that the verification query on the abstract network
constitutes an over-approximation of the original query.  Finally,
this pair-wise merging is then repeated, resulting in a much smaller
network.

An example showing the result of applying this abstraction to the network from
Fig.~\ref{fig:runningExample} appears in
Fig.~\ref{fig:runningExample-cegarnn-abstraction}: $n_{2,1}$ and
$n_{2,2}$ both already belong to the same category (in this tiny network, 
all other categories are empty). We merge them into a single abstract 
neuron, $n_{2,1+2}$. The output weights are aggregated by a sum operation, 
so we get an output weight of $3+4=7$.  For the input weight aggregation, 
we take the maximal value of the two, so we get an input weight of
$\max(10,1)=10$. The key property here is that for every input $x$, $N'(x)\geq N(x)$.  For additional details, as well as a
discussion of various heuristics for selecting which neurons should be
merged and in what order, see~\cite{ElGoKa2020, LiXiShSoXuMi2022}.

\begin{figure}[htp]
	\begin{center}
		\scalebox{1} {
			\def\layersep{2.5cm}
			\begin{tikzpicture}[shorten >=1pt,->,draw=black!50, node distance=\layersep,font=\footnotesize]
				
				\node[input neuron] (I-1) at (0,-1) {$n_{1,1}$};
				
				\path node[hidden neuron] (H-1)
				at (\layersep,-1 cm) {$n_{2,1+2}$};
				
				\node[output neuron] at (2*\layersep, -1) (O-1) {$n_{3,1}$};
				
				\draw[nnedge] (I-1) -- node[above] {$10$} (H-1);
				\draw[nnedge] (H-1) -- node[above] {$7$} (O-1);
				
				\node[annot,above of=H-1, node distance=1cm] (hl) {Hidden layer};
				\node[annot,left of=hl] {Input layer};
				\node[annot,right of=hl] {Output layer};
			\end{tikzpicture}
		}
		\caption{An abstract network, generated from the
			network from Fig.~\ref{fig:runningExample}.}
		\label{fig:runningExample-cegarnn-abstraction}
	\end{center}
\end{figure}
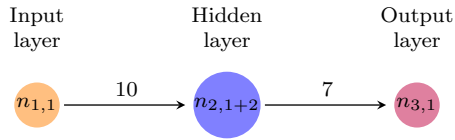

The refinement operation is then carried out by splitting an abstract
neuron, which represents a set of original neurons, into two or more
new neurons. In the example, the refinement step would be to split the
abstract neuron $n_{2,1+2}$ into the original neurons $n_{2,1}$ and
$n_{2,2}$, and to restore their original weights.

Empirical evaluations of \cegarnn{} demonstrated its great potential to
enhance the scalability of NNV engines, but as many experiments
show~\cite{ElGoKa2020}, there is much room for improvement.

\section{\cegartwo{}: Tighter Abstract Queries}\label{sec:cegartwo}
\subsection{Motivation}

Although \cegarnn{} is quite useful in many cases, it is also prone to
producing spurious counter-examples, which in turn triggers multiple
refinement steps that slow the process down. For example, observe
again the network $N$ from Fig.~\ref{fig:runningExample}, and consider
the verification query
\[
v_1 = \langle N,n_{1,1}\in[20,21],n_{3,1}>800 \rangle
\]
Here, a sound verifier will declare that $v_1$ is \unsat{}, because
for inputs in the range $[20,21]$, network $N$ can only produce
outputs that are upper-bounded by 
$N(21)=714$. If we attempt to verify this query using \cegarnn{}, we
would generate the abstract query
\[
\langle N', n_{1,1}\in[20,21], n_{3,1}>800 \rangle,
\]
where $N'$ is the network in
Fig.~\ref{fig:runningExample-cegarnn-abstraction}. For this query, a
sound verifier will return a satisfying assignment, such as
$n_{1,1}=20$. Of course, this assignment is spurious: although
$N'(20)=1400\geq800$, we get that $N(20)=680<800$. Thus, refinement
would be carried out, transforming $N'$ back into the original network
$N$, and the overall process would be slower than just verifying $N$
directly. More broadly, we recognize the following issue with \cegarnn{} and
related techniques:

\mysubsection{Performance vs Accuracy.}  Neuron-merging-based
abstraction techniques, such as \cegarnn{}, have an intrinsic
trade-off between performance and accuracy. In order to avoid coarse
model abstractions and the ensuing spurious counter-examples, it is
desirable to look for \textit{accurate} model
abstractions. On one hand, a common approach in CEGAR-based
verification is to heuristically guess an accurate initial
abstract model, so that future satisfying assignments will not be
spurious; and if a spurious assignment is discovered nonetheless, to
try and heuristically select a refinement operation that will restore
as much accuracy as possible. On the other hand, a model with a higher
accuracy is almost always larger and hence verifying it is slower.
Thus, these two requirements conflict with each other:
generating an \textit{accurate} abstract model restricts our ability
to generate small abstractions, and results instead in larger
models that take longer to verify.

\subsection{Introducing Property Abstractions}
In order to overcome the aforementioned issue, we introduce an
extension to CEGAR, which we term \cegartwo{}: CEGAR Enhanced by TighTEning.  In \cegartwo{}, when we are given a
verification query $\langle N, P, Q\rangle$, instead of abstracting
and refining only the network $N$, we may also alter the output
property $Q$ in order to produce an over-approximate query
$\langle N', P, Q'\rangle$.

To set the stage, we introduce the following definitions:
\begin{definition}	
	A verification query $\langle N',P,Q'\rangle $ is an
	\textit{over-approximation} of another verification query
	$\langle N,P,Q\rangle$, if and only if the unsatisfiability of
	$\langle N',P,Q'\rangle $ implies the unsatisfiability of
	$\langle N,P,Q\rangle$.
\end{definition}
We refer to the process in which $\langle N',P,Q'\rangle $ is created
from $\langle N,P,Q\rangle $ as \emph{query abstraction}.

\begin{definition}
	Let $\langle N,P,Q\rangle$ be some base verification query.
	A query
	$\langle N'',P,Q''\rangle $ is called a refinement of
	a query
	$\langle N',P,Q'\rangle$, if
	\begin{inparaenum}[(i)]
		\item
		$\langle N'',P,Q''\rangle $ is an over-approximation of
		$\langle N,P,Q\rangle$; and
		\item
		$\langle N',P,Q'\rangle $ is an over-approximation of
		$\langle N'',P,Q''\rangle$.
	\end{inparaenum}
\end{definition}

To illustrate the effect of changing the property, we consider again
the verification query from our running example:
$v_1=\langle N,n_{1,1}\in[20,21],n_{3,1}>800 \rangle$.  Let us
consider what happens to this query if we change its output property.
Decreasing the constant $800$, for example by setting $ n_{3,1}>400 $,
renders the property easier to satisfy (e.g., the output 600 satisfies
the latter output property, but not the former).  Therefore, if
$\langle N,n_{1,1}\in[20,21],n_{3,1}>400 \rangle$ is \unsat{}, then
$v_1$ is also \unsat{}.  Consequently, we claim that decreasing the
constant that appears in an output property $Q$ results in an
over-approximation.

While the aforementioned process is sound, it goes against
the grain of our desired approach: decreasing the constant in the output
property could potentially result in additional, not fewer, spurious
counter-examples.  Thus, what we would like to do is to
\emph{increase} the constant that appears in the output property, in
order to rule out spurious counter-examples.  More formally, there is
a set of spurious inputs $S$, whose outputs satisfy the property in
the abstract network: $\forall x\in S: (N'(x)>c)$, but not in the
original network, where their outputs are smaller:
$\forall x\in S: (N(x)\leq c)$. By increasing the constant in the
output property, we seek to avoid these spurious examples.

As it turns out, there are cases in which we can increase the
output bound, and still obtain an over-approximate query. Consider the
query 
\[
v_2=\langle N',n_{1,1}\in[20,21],n_{3,1}>1486 \rangle
\]
In
$v_2$, there are changes (with respect to $v_1$) in both the model $N$
and the output property $Q$.  By applying changes to the model $N$ and
generating $N'$, the output increases: for every input $x$,
$N(x) \leq N'(x)$. In other words, there is a minimal (non-negative)
difference between the outputs:

\begin{equation}\label{eq:abs-min-diff}
	\exists d\geq0: \forall x : N(x) + d \leq N'(x)  
\end{equation}
Given that the input property $P:=x\in[20,21]$ was not changed, and
assuming that we can calculate $d$, we can \textit{increase} the
output constant by any number in $[0,d]$, and still get an over-approximate
query; if the value in the abstract network is smaller than $c+d$, the
output of the original network is bounded from above by $(c+d)-d=c$
(by Eq.~\ref{eq:abs-min-diff}).

Going back to our running example, we need to calculate the minimal
difference between the respective outputs of the networks, for any
input vector in the range specified by property $P$. For $P=[20,21]$,
it can be shown that the minimal difference is bounded by
$d\geq N'(20)-N(21)=1400-714=686$.  Consequently, we can increase the
constant in the output property from 800 (in $v_1$) to $800+686=1486$
(in $v_2$), and still get an over-approximate query: if
$\forall x\in[20,21]: N'(x) \leq 1486$, then (by difference of bounds)
$\forall x\in[20,21]: N(x) \leq 1486-686=800$.  Making this adjustment
rules out the spurious counter-example we saw before, namely
$N'(20)=1400$.

More broadly, the basic idea underlying \cegartwo{} is
that in order to produce a small but accurate abstraction,
the output property ($Q$) should be tightened as much as possible in
parallel to reducing the size of the neural network ($N$). 
The abstraction became more refined, and as a consequence, the counter example is more relevant, and the number of refinement steps should decrease.
Naturally, the abstraction and tightening processes are linked, as the network determines the possible output properties. 

The example shows how changing $N$ enables
tightening $Q$, but \cegartwo{} is not limited to a specific order;
it is also possible to first increase the constant of $Q$ and only
then change $N$, as long as the over-approximation property is maintained.

Algorithm \ref{alg:abVerification2} shows the general outline of
the \cegartwo{} framework.  First, we generate an initial \textit{abstract
	verification query} using \queryAbstractOp{}.  Then a loop starts,
where in each iteration we verify the current abstract query.  If the
answer is \unsat{}, we are guaranteed that the original query is also
\unsat{} (by the over-approximation property), and can stop and return
\unsat{}. Otherwise, the satisfying assignment $x_0$ is examined in
the original model, and if it is also a satisfying assignment there,
we return \sat{} and $x_0$.  In the case where $x_0$ is not a
satisfying assignment for $\langle N,P,Q\rangle$, the current abstract
query is apparently too coarse and is thus refined using
\queryRefineOp{} --- producing a more precise abstract query, for the
next iteration.

\begin{algorithm}
	\caption{\cegartwo-based Verification($N,P,Q$)}
	\begin{algorithmic}[1]
		\label{alg:cegar2}
		\STATE $\langle N', P, Q'\rangle\leftarrow$
		\queryAbstractOp{$(\langle N, P, Q\rangle)$}
		\label{step:query-abstraction}
		\IF {\emph{Verify}$\langle N', P, Q'\rangle$ is \unsat{}} \label{step:verify-abstract-query}
		\STATE return \unsat{}
		\ELSE
		\STATE Extract counterexample $x_0$
		\IF {$x_0$ is a counterexample for $\langle N, P, Q\rangle$}
		\STATE return \sat{}, $x_0$
		\ELSE
		\STATE $\langle N'',P,Q''\rangle\leftarrow$ \queryRefineOp{$(\langle N',P,Q'\rangle, N, x_0)$}   \label{step:query-refinement}
		\STATE $\langle N',P,Q'\rangle\leftarrow \langle N'',P,Q''\rangle$
		\STATE Goto step~\ref{step:verify-abstract-query}
		\ENDIF
		\ENDIF
	\end{algorithmic}
	\label{alg:abVerification2}
\end{algorithm}

The structure of \cegartwo{} is similar to that of \cegarnn{}, but
instead of invoking the \abstractOp{} and \refineOp{} operations as in
Algorithm \ref{alg:abVerification}, which only abstract and refine the
network, \cegartwo{} invokes \queryAbstractOp{} and \queryRefineOp{},
which abstract and refine both the network and the output property.  Therefore,
Algorithm~\ref{alg:abVerification} is a special case of
Algorithm~\ref{alg:abVerification2}, and \cegartwo{} extends CEGAR.

\section{DNN Verification Using \cegartwo{}}\label{sec:cegarnntwo}
In this section, we describe \cegarnntwo{}, which is our implementation
of \cegartwo{} for NNV. Specifically, we propose a particular
implementation of the \queryAbstractOp{} and \queryRefineOp{}
operators, which modify the model and the output property in a way
that soundly produces over-approximations of the original query.

\sloppy
Given a verification query $\langle N,P,Q\rangle$, our proposed
implementation of \queryAbstractOp{} is shown as
Algorithm~\ref{alg:cegarnntwo-query-abstract-op}. It begins by
generating $N'$, an abstract neural network, using the same
\abstractOp{} from the \cegarnn{} framework; and then proceeds to
tighten the output property, by adding to the output constant a scalar
$d$, which lower-bounds the minimal difference in outputs between $N'$
and $N$.

\begin{algorithm}
	\caption{AbstractQueryGeneration($N,P,Q$)}
	\begin{algorithmic}[1]
		\STATE $N'\leftarrow \abstractOp{}(N)$
		\STATE $Q' \leftarrow$ TightenProperty($N',N,P,Q$) 
		\STATE return $\langle N',P,Q' \rangle$
	\end{algorithmic}
	\label{alg:cegarnntwo-query-abstract-op}
\end{algorithm}

In Algorithm~\ref{alg:under-approximate-property} we describe how to
compute the constant $d$. This is performed by calculating the lower
bound of the abstract network's output, and the upper bound of the
original network's output, and then subtracting the latter from the
former. If the result is positive, it can be used in order to update
the output property in the over-approximate verification query. The
output property is of the form $Q:=y<c$ for some constant $c$ (denoted
by $Q_c$ in Algorithm~\ref{alg:under-approximate-property}), and so
the update is performed by setting $Q':=y<c+d$.

\begin{algorithm}
	\caption{TightenProperty($N',N,P,Q$)}
	\begin{algorithmic}[1]
		\STATE Compute $l_{N'}$, a lower bound on the output of $N'$
		\STATE Compute $u_{N}$ an upper bound on the output of $N$
		\STATE $ d = \max (0, l_{N'}-u_{N}) $
		\STATE {$Q':= y > Q_c + d$}
		\RETURN $Q'$
	\end{algorithmic}
	\label{alg:under-approximate-property}
\end{algorithm}

Computing the lower and upper bounds of the abstract and the original networks with respect to a given input range in Algorithm \ref{alg:under-approximate-property} is based on \textit{bound propagation methods}, which maintain and propagate tractable and sound bounds through neural networks. Bound propagation has been studied extensively, and there are many scalable methods intended for this purpose~\cite{GoDvStBuQiUeArMaKo2018,KrDvStGoMaKo2018,GeMiDrTsChVe18,SiGePuVe2019,WaPeWhYaJa2018,WoKo2017,KaZhHuYiKaMiBhXuCh2020,ZhWeChHsDa2018,WaZhXuLiJaHsKo2021}.
We give a simple example later on.

\begin{lemma}\label{lemma:cegarnntwo-abstract}
	Algorithm \ref{alg:cegarnntwo-query-abstract-op} returns an
	over-approximation of the verification query passed to it as input.
\end{lemma}

\begin{proof}
	Algorithm~\ref{alg:cegarnntwo-query-abstract-op} begins by
	generating $N'$, an abstract network, from $N$. Then it invokes
	TightenProperty, described in
	Algorithm~\ref{alg:under-approximate-property}, which sets
	$Q':=y'\leq c+max(0,l_{N'}-u_N)$, where $u_N$ is an upper bound for
	the output of $N$ and $l_{N'}$ is a lower bound for the output of
	$N'$. We now wish to prove that if $\langle N',P,Q'\rangle$ is
	\unsat{}, then $\langle N,P,Q\rangle$ is also \unsat{}. Differently
	put, we need to show that 
	\[
	\forall x\in P: N'(x)\leq c+max(0,l_{N'}-u_N) \Rightarrow \forall x\in P:
	N(x)\leq c
	\]
	This is equivalent, modus tollens, to proving that
	\[
	\exists x\in P: N(x) > c \Rightarrow \exists x\in P: N'(x) >
	c+max(0,l_{N'}-u_N)
	\]
	The last implication holds since we know that for every input $x$,
	the output increases after abstracting $N$ to $N'$ by at least
	$l_{N'}-u_N$; and we know also that the output cannot
	decrease. Overall, model abstraction increases $x$'s output by at
	least $max(0,l_{N'}-u_N)$, and hence the implication holds.
\end{proof}

Next, \queryRefineOp{} is implemented in Algorithm
\ref{alg:cegarnntwo-query-refine-op}. First, we refine the previous
abstraction of the network (by splitting neurons that had previously
been merged), and then computing a new output property, from scratch,
with respect to this new network.

\begin{algorithm}
	\caption{RefinedQueryGeneration($\langle N',P,Q'\rangle,N,x_0$)}
	\begin{algorithmic}[1]
		\STATE $N''\leftarrow \refineOp(N',N,x_0)$
		\STATE $Q'' =$ TightenProperty($N'',N,P,Q$) 
		\STATE return $\langle N'',P,Q'' \rangle$
	\end{algorithmic}
	\label{alg:cegarnntwo-query-refine-op}
\end{algorithm}

\begin{lemma}\label{lemma:cegarnntwo-refine}
	Algorithm~\ref{alg:cegarnntwo-query-refine-op} returns an 
	over-approximation of the verification query $\langle N,P,Q\rangle$.
\end{lemma}

\begin{proof}
	$N''$ is an abstract network generated from $N$. The remainder of the proof is the
	same as in the proof of Lemma~\ref{lemma:cegarnntwo-abstract}.
\end{proof}

Now that we have established the soundness of our approach via
Lemmas~\ref{lemma:cegarnntwo-abstract}
and~\ref{lemma:cegarnntwo-refine}, we proceed to prove that it
always terminates.

\begin{lemma}\label{lemma:cegarnntwo-convergence}
	For any verification query $\langle N,P,Q \rangle$, \cegarnntwo{}
	converges.
\end{lemma}
\begin{proof}
	\cegarnntwo{} implements Algorithm~\ref{alg:abVerification2}, using
	\queryAbstractOp{} from
	Algorithm~\ref{alg:cegarnntwo-query-abstract-op} and
	\queryRefineOp{} from
	Algorithm~\ref{alg:cegarnntwo-query-refine-op}. We show that with
	these implementations, Algorithm~\ref{alg:abVerification2}
	converges.
	
	Assume that Algorithm~\ref{alg:abVerification2} does not converges, then Line 11 takes place infinite number of times. In that case, Lines 9-10 takes place infinite number of times too. Notice that after the abstraction in Line 1, the network $N'$ only changes in Line 10. After finite number of calls to \queryRefineOp{} which is implemented by Algorithm~\ref{alg:cegarnntwo-query-refine-op}, the query $\langle N',P,Q' \rangle$ will be equal to the original query $\langle N,P,Q \rangle$;
	Algorithm~\ref{alg:under-approximate-property} refine the network and then tighten the query. When $ N'$ is fully refined ($ N' == N $), the difference between the lower bound of $N'$ and the upper bound of $N$ can't be positive and $d=0$ which implies that $ Q'==Q $. In the next iteration, both getting \unsat{} or \sat{} will terminate the verification, since the current query is equal to the original query. In contradiction to our assumption.
	
\end{proof}

Returning to our running example, given the verification query $v_1$
above, applying Algorithm~\ref{alg:cegar2} triggers Algorithm
\ref{alg:cegarnntwo-query-abstract-op}, which generates $N'$ from $N$;
and then, in turn, triggers
Algorithm~\ref{alg:under-approximate-property}, which calculates $d$,
which bounds the minimal difference between $N$ and $N'$. For
simplicity, we use a na\"ive method for calculating the lower and upper bounds of a network,
called interval bound propagation, or
\IBP{}~\cite{GoDvStBuQiUeArMaKo2018}. 

In \IBP{}, bounds are propagated forward in the network, starting from
the input layer, layer by layer, to the output layer. The method
assumes that lower and upper bounds for the input neurons are provided
a priori, and then uses a linear combination of the bounds of neurons
from one layer to compute lower and upper bounds for neurons in the
following layer. The linear combination is decided according to the
weights of the edges that connect the two layers.

As an example, we show how to compute the bounds for neuron $n_{3,1}$
in the original network $N$. The range of the input neuron $n_{1,1}$
is $[20,21]$. The range of the possible values of $n_{2,1}$ is
$[20,21]\cdot10=[200,210]$, since its only input edge comes from
$n_{1,1}$ and its weight is 10. Similarly, the possible values of $n_{2,2}$
are $[20,21]$ since its only input edge comes from $n_{1,1}$ and
its weight is 1. Moving to the output layer, the range of possible
values of $n_{3,1}$ is calculated using the weighted sum of the ranges
of $n_{2,1}$ and $n_{2,2}$: the lower bound is
$200\cdot3+20\cdot4=680$, and the upper bound is
$210\cdot3+21\cdot4=714$.  Similarly, by propagating the input range
to the output range in $N'$ we get that the output range of $n_{3,1}$
in $N'$ is $[1400,1470]$.
Therefore, the minimal difference is bounded by
$d=N'(20)-N(21)=1400-714=686$. Now, $d$ can be used in updating the
verification query into the over-approximate verification query
$\langle N', n_{1,1}\in[20,21],n_{3,1}<800+686=1486\rangle$, which is
exactly $v_2$ above.

\section{Implementation and Evaluation}
\label{sec:evaluation}

\sloppy \mysubsection{Implementation.} For evaluation purposes, we
created a proof-of-concept implementation of \cegartwo{}, referred to
as \cegarnntwo{}, where the \queryAbstractOp{} and \queryRefineOp{}
operations are implemented according to
Algorithms~\ref{alg:cegarnntwo-query-abstract-op}
and~\ref{alg:cegarnntwo-query-refine-op}, respectively. The neuron
merging and splitting operations, \abstractOp{} and \refineOp{}, were
taken from the publicly released code of~\cite{ElCoKa2022}, which
implements the methods proposed in~\cite{ElGoKa2020}.

As our tool's underlying verification engine, we used the Marabou
framework~\cite{KaHuIbJuLaLiShThWuZeDiKoBa19Marabou} (although other
tools could be used, instead). Marabou is a sound and complete
verifier~\cite{KaHuIbJuLaLiShThWuZeDiKoBa19Marabou, WuOzZeIrJuGoFoKaPaBa20}
that runs a Simplex solver at its core~\cite{KaBaDiJuKo17Reluplex,
	KaBaDiJuKo21}, combined with abstract-interpretation
enhancements~\cite{WaPeWhYaJa2018, SiGePuVe2019, ElGoKa2020, OsBaKa22,
	ZeWuBaKa22}, proof-production capabilities~\cite{IsBaZhKa22}, 
advanced splitting heuristics~\cite{WuZeKaBa22}, optimization
techniques~\cite{StWuZeJuKaBaKo21}, and support for various activation
functions~\cite{AmWuBaKa21}. The framework has previously been applied
to perform various tasks, such as DNN repair~\cite{ReKa22,
	GoAdKeKa20}, explainability~\cite{BaKa23}, reinforcement learning
verification~\cite{AmCoYeMaHaFaKa23, AmScKa21, ElKaKaSc21, AmMaZeKaSc23},
DNN simplification~\cite{LaKa21, GoFeMaBaKa20}, and DNN
ensemble selection~\cite{AmZeKaSc22} and industrial needs~\cite{AmFrKaMaRe23}.

For the lower and upper bound computation performed in
Algorithm~\ref{alg:under-approximate-property}, we implemented the
interval bound propagation (IBP) method proposed
in~\cite{GoDvStBuQiUeArMaKo2018}; and also used the symbolic bound
tightening (SBT) method~\cite{WaPeWhYaJa2018} and the DeepPoly
method~\cite{SiGePuVe2019}, both of which were already implemented as
part of the Marabou engine~\cite{KaHuIbJuLaLiShThWuZeDiKoBa19Marabou}.
Our tool is implemented in Python and is publicly available
online, along with all benchmarks used in our evaluation.\footnote{\url{https://github.com/yizhake/cegarette_nn}}
All experiments reported below
were conducted on x86-64 Gnu/Linux-based machines, using a single
Intel(R) Xeon(R) Gold 6130 CPU @ 2.10GHz core.

\mysubsection{Benchmarks.}
We conducted extensive experiments using two sets of benchmarks:
ACAS-Xu~\cite{JuLoBrOwKo2016} and MNIST~\cite{Le98}.

ACAS-Xu is a set of 45 DNNs intended to operate as an airborne
collision avoidance system. Each of these networks receives sensor
information regarding the aircraft's trajectory and velocity, as well
as those of other aircraft nearby; and produces a horizontal turning
advisory, intended to reduce the chance of airborne collision. Each of
these networks consists of an input layer with 5 neurons, followed by
6 hidden layers with 50 neurons each, and a final output layer with 5
additional neurons --- yielding a total of 310 neurons.

For specifications, we used adversarial robustness queries, which are
the de facto standard for DNN verification~\cite{BaLiJo21}. Each such
query specifies an input point $x_0$ and a radius $\delta$ and
states that any point within the $\delta$-ball around $x_0$ must
produce the same classification as $x_0$. Here, we used 20
previously proposed adversarial robustness properties~\cite{ElGoKa2020}, each with a 2-hour timeout.

For the second family of benchmarks, we used the MNIST dataset of
grayscale images of hand-written digits between 0 and 9. We used 60000
images to train 3 different networks, whose topologies are listed in
Table~\ref{table:mnist-trained-nets}. These networks achieved high
accuracy rates (although not as high as the state of the art~\cite{Ma21}) on 10000 test images, as detailed in the ``Accuracy'' row
in Table \ref{table:mnist-trained-nets}. For these networks, we again
used adversarial robustness queries.  Specifically, we selected
30 input points that were
sampled uniformly at random. For $\delta$, we used the values
$ \delta \in \{1e^{-3}, 1e^{-2}, 2e^{-2}, 5e^{-2}, 7e^{-2}, 9e^{-2},
1e^{-1}\}$.  For each candidate query, i.e.,~a pair of input point
$x_0$ and a radius
$\delta$, we sampled 10000 random inputs in the
$\delta$-ball around
$x_0$, and ensured that they were all correctly classified.
If any random sample was misclassified, we discarded the query without
verification  (this was done in order to filter out very simple
queries, where formal verification is not needed).  After this
filtering, the total number of queries remaining for each of the
networks was 653, 797, and 560. Because these networks were more
complex than the ACAS Xu case, we set an arbitrary timeout value of 3 days for
each query.  The actual encoding as a Marabou query was performed
according to the definition of \emph{standard robustness} (Definition
2,~\cite{CaKoDaKoKaAmRe22}).

\begin{table}
	\setlength\tabcolsep{0pt}
	\caption{Network sizes and properties.}
	\label{table:mnist-trained-nets}
	\begin{tabular*}{\linewidth}{@{\extracolsep{\fill}} l *{4}{S[table-format=2.1]} }
		\toprule   
		
		& \textbf{Acas Xu} 
		& \textbf{MNIST 1}  
		& \textbf{MNIST 2}
		& \textbf{MNIST 3} \\
		\midrule
		Inputs & 5 & 784 & 784 & 784 \\
		Outputs & 5 & 10 & 10 & 10 \\
		Hidden layers & 6 & 5 & 6 & 15 \\
		Total hidden neurons & 300 & 144 & 320 & 224 \\
		Accuracy & & 96.94 & 97.13 & 95.43 \\
		\bottomrule
	\end{tabular*}
\end{table}

\mysubsection{Evaluation.}  Recall that our approach depends on our
ability to compute tight lower and upper output bounds as part of
Algorithm~\ref{alg:under-approximate-property}. Multiple methods have
been proposed for computing such bounds, with varying degrees of
accuracy --- and with the more accurate ones typically taking longer
to run. Thus, in our first experiment, we set out to compare the
different bound propagation approaches, namely \IBP{}, \SBT{}, and
\DeepPoly{}, and measure their usefulness as part of our framework.

Table~\ref{table:compare-bound-propagation-methods} depicts the
results obtained using the three approaches when applied to the
ACAS-Xu benchmarks. \IBP{}, which is the most lightweight but also the
least precise among the approaches, performed the worst --- leading to
the highest number of timeouts. This indicates that the bounds it
computed were fairly loose, triggering large sequences of spurious
examples and refinement steps, eventually leading to the
timeouts. \DeepPoly{}, which is the most precise among the three but
also the most computationally expensive, achieved better bounds, and
consequently fewer timeouts. \SBT{}, which is not as precise as
\DeepPoly{} but which is quicker to run, obtained the best
results. Thus, we selected \SBT{} as the best configuration for our
tool and used it in the remaining experiments.

\begin{table}
	\caption{Comparing Bound propagation methods on the ACAS-Xu benchmarks.}
	\label{table:compare-bound-propagation-methods}
	\begin{center}
		\begin{tabular}{ 
				| >{\centering\arraybackslash}X m{5em} 
				| >{\centering\arraybackslash}X m{5em} 
				>{\centering\arraybackslash}X m{5em} 
				>{\centering\arraybackslash}X m{5em} 
				|
			} 
			\hline
			& \IBP{} & \SBT{} & \DeepPoly{} \\ \hline
			\#Finished & 397 & \textbf{851} & 833 \\ 
			\#Timeouts & 463 & \textbf{49} & 67 \\ 
			\hline
		\end{tabular}
	\end{center}
\end{table}

In our second experiment, we set out to measure the overall improvement
afforded by our approach. Since it has already been established that
abstraction-refinement often improves over direct
verification~\cite{ElCoKa2022, ElGoKa2020, AsHaKrMu20, PrAf2020}, and
because our approach extends the \cegarnn{} approach~\cite{ElGoKa2020}, we used \cegarnn{} as our baseline.

The results of comparing the two tools on all benchmark sets are
displayed in Table~\ref{fig:exp-2-results}. Most importantly, the
results in the \emph{Timeouts} and \emph{Finished} columns confirm that
\cegarnntwo{} significantly improves over \cegarnn{} in terms of
finished experiments: a total of 2674 for \cegarnntwo{}, versus 1618
for \cegarnn{} --- a 65\% improvement. Because both tools use the same
underlying verifier, and apply the same basic abstraction and
refinement operators, this improvement stems directly from the
property abstraction mechanism, and its ability to reduce the number
of spurious counter-examples.

The next column, \emph{Faster Verification Time}, counts the number of instances in which one tool outperformed the other. It shows that \cegarnntwo{} achieves an
improvement of 12.85\% over \cegarnn. Finally, the
\emph{Fewer Refinement Steps} column counts the number of experiments
in which the total number of refinement steps was
smaller, and shows that \cegarnntwo{} achieves an improvement of
586.25\%.
We note that all comparison metrics that we used (except
\emph{Timeouts}) consider only those benchmarks that both tools
successfully solved. We also mention that, as with all SMT solvers,
slight changes to the input can result in widely different search
paths and runtimes, as demonstrated, e.g.,  in the case of the MNIST-3 network.

The graphs in Figure \ref{fig:cegarrete_vs_cegar_over_time} show how many experiments were completed by each method at different time points. ACAS-Xu/MNIST (MNIST-1, MNIST-2 and MNIST-3 together) results are represented in the left/right graph respectively, and indicate that \cegarnntwo{} is close in performance to \cegarnn{} at short time constants, while achieving a significant advantage when the threshold time is extended.

\begin{table}
	\caption{Comparing \cegarnntwo{} and \cegarnn{}.}
	\label{fig:exp-2-results}
	\resizebox{\columnwidth}{!}{%
		\begin{tabular}{
				>{\centering\arraybackslash}X m{5cm}
				>{\centering\arraybackslash}X m{1.8cm}
				>{\centering\arraybackslash}X m{1.8cm}
				>{\centering\arraybackslash}X m{2.2cm}
				>{\centering\arraybackslash}X m{1.7cm}
			}
			\toprule
			& \textbf{Timeout} & \textbf{Finished} & \makecell{\textbf{Faster} \\ \textbf{Verification}\\ \textbf{Time}} & \makecell{\textbf{Fewer} \\ \textbf{Refinement}\\ \textbf{Steps}} \\
			\midrule
			\begin{tabular}{m{2cm}m{2cm}} \textbf{ACAS-Xu} & 
				\begin{tabular}{c}\cegarnn{} \\ \cegarnntwo{}\end{tabular} 
			\end{tabular} & 
			\begin{tabular}{c}501 \\ \textbf{49}\end{tabular} & 
			\begin{tabular}{c} 397 \\ \textbf{851}\end{tabular} & 
			\begin{tabular}{c}\textbf{272} \\ 115\end{tabular} & 
			\begin{tabular}{c}\textbf{73} \\ 2\end{tabular} \\
			\midrule
			\begin{tabular}{m{2cm}m{2cm}} \textbf{MNIST-1} & 
				\begin{tabular}{c}\cegarnn{} \\ \cegarnntwo{}\end{tabular} 
			\end{tabular} & 
			\begin{tabular}{c}268 \\ \textbf{88}\end{tabular} & 
			\begin{tabular}{c}396 \\ \textbf{576}\end{tabular} & 
			\begin{tabular}{c}127 \\ \textbf{252}\end{tabular} & 
			\begin{tabular}{c}1 \\ \textbf{283}\end{tabular} \\
			\midrule
			\begin{tabular}{m{2cm}m{2cm}} \textbf{MNIST-2} & 
				\begin{tabular}{c}\cegarnn{} \\ \cegarnntwo{}\end{tabular} 
			\end{tabular} & 
			\begin{tabular}{c} 716\\ \textbf{228}\end{tabular} & 
			\begin{tabular}{c} 230 \\ \textbf{718}\end{tabular} & 
			\begin{tabular}{c}83 \\ \textbf{134}\end{tabular} & 
			\begin{tabular}{c}0 \\ \textbf{133}\end{tabular} \\
			\midrule
			\begin{tabular}{m{2cm}m{2cm}} \textbf{MNIST-3} & 
				\begin{tabular}{c}\cegarnn{} \\ \cegarnntwo{}\end{tabular} 
			\end{tabular} & 
			\begin{tabular}{c}\textbf{101} \\ 167\end{tabular} & 
			\begin{tabular}{c} \textbf{595} \\ 529\end{tabular} & 
			\begin{tabular}{c}195 \\ \textbf{263}\end{tabular} & 
			\begin{tabular}{c}6 \\ \textbf{51}\end{tabular} \\
			\midrule
			\begin{tabular}{m{2cm}m{2cm}} \textbf{TOTAL} & 
				\begin{tabular}{c}\cegarnn{} \\ \cegarnntwo{}\end{tabular} 
			\end{tabular} & 
			\begin{tabular}{c}1506 \\ \textbf{532}\end{tabular} & 
			\begin{tabular}{c} 1618 \\ \textbf{2674}\end{tabular} & 
			\begin{tabular}{c}677 \\ \textbf{764}\end{tabular} & 
			\begin{tabular}{c}80 \\ \textbf{469}\end{tabular} \\
			\midrule
		\end{tabular}%
		}
	\end{table}

    \begin{figure}[!ht]
		\centering
		\scalebox{0.5}{
  		\begin{tabular}{cc}
			\includegraphics[valign=T]{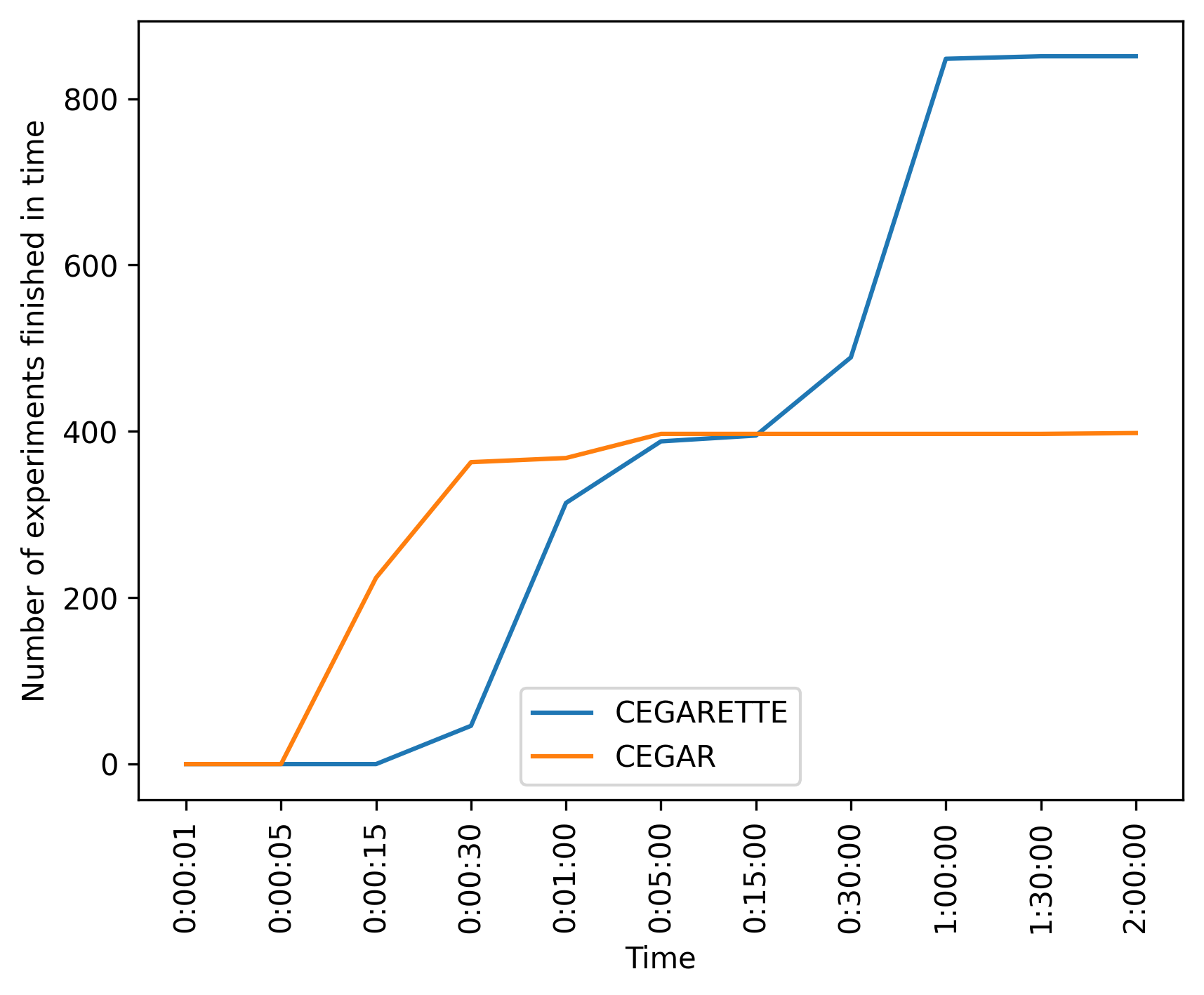} &  
			\includegraphics[valign=T]{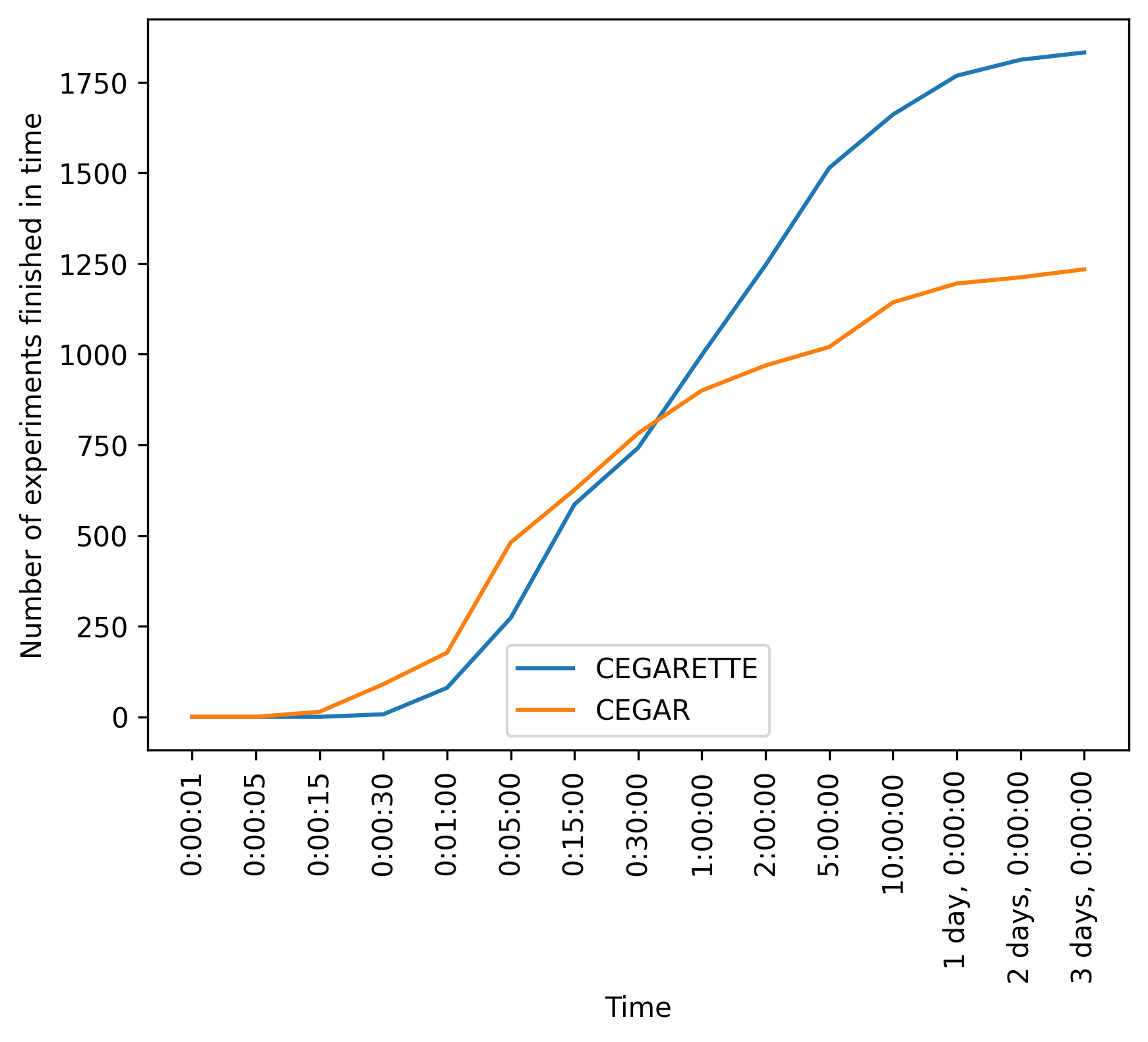}  \\
		\end{tabular}			
		}
		\caption{Comparing the numbers of finished experiments over time between \cegarnntwo{} and \cegarnn{} on ACAS-Xu (left) and MNIST (right) datasets.}
		\label{fig:cegarrete_vs_cegar_over_time}
	\end{figure}

\section{Related Work}
\label{sec:relatedWork}

There are two main approaches to validating the robustness of deep
neural networks.  The first combines methods of dynamic analysis
and heuristic search to find violations and test the system (e.g.~\cite{TiPeJaRa2017,PeCaYaJa2019,ReZhQiLi2020,CaWa2016}). The
second, which we follow in this work, is that of formal
verification, where correctness is established using a rigorous,
automated procedure~\cite{LiArLaBaKo19}.

Several DNN verification techniques have been studied in recent years,
based on multiple approaches: SMT-based techniques, such as
Marabou~\cite{KaHuIbJuLaLiShThWuZeDiKoBa19Marabou},
Reluplex~\cite{KaBaDiJuKo17Reluplex}, and
others~\cite{KaBaKaSc19,HuKwWaWu17}; techniques based on mixed integer
programming, including Planet~\cite{Ehlers2017},
MIPVerify~\cite{TjXiTe19} and
others~\cite{DuChJhSaTi2019,BuTuToKoKu17,DuJhSaTi18,BuLuTuToKoKu2019}; symbolic
interval propagation techniques~\cite{WaPeWhYaJa2018}; abstract
interpretation techniques~\cite{GeMiDrTsChVe18}; and many others
(e.g.,~\cite{AnPaDiCh19, JaBaKa20, KuKaGoJuBaKo18, LoMa17,
	NaKaRySaWa17, WuOzZeIrJuGoFoKaPaBa20, XiTrJo18}).  Our approach can
be used to extend various abstraction and refinement mechanisms, and be
integrated with many sound and complete DNN verifiers as backends.
Incomplete verification techniques could also be used as part of our
approach, potentially improving performance but at the cost of
incompleteness; we leave this for future work.

There have been multiple attempts to utilize abstract
interpretation~\cite{CoCo1977} to decrease the complexity of DNN
verification, and expedite the verification
process~\cite{GeMiDrTsChVe18, SiGePuVe2019,
	YaLiLiHuWaSuXuZh2020}. These methods capture the behavior of
propagated values in the network using abstract domains such as 
boxes, zonotopes, or polyhedra. Often, these  methods 
rely on coarse over-approximations, and are incomplete; although
various refinement methods, as well as integration with complete
verifiers, have been proposed to mitigate this.

Apart from the DNN abstraction technique
that we leveraged here~\cite{ElGoKa2020}, other incomplete
abstraction techniques have been proposed~\cite{AsHaKrMu20}.  These
manipulate the neurons and the edges of the network, using semantic
similarity; utilize clustering methods in order to merge similar
neurons; or merge neurons and compute ranges of weights and biases
that the merged neurons can take. Integrating our framework with these
additional approaches should be possible, and is left for future work.

A few recent papers proposed to instrument CEGAR-based approaches for DNN verification~\cite{ElGoKa2020,LiXiShSoXuMi2022}.
The network is preprocessed such that, later on, multiple neurons can
be merged while causing a strict increase in the output values, hence
over-approximating the original network. Refinement is done by
splitting past-merged neurons.
Both are oblivious to the underlying verification engine as long as it is sound and complete. 

The work in~\cite{ElCoKa2022} uses residual reasoning to optimize CEGAR-based approaches, given that the backend applies case splitting.
Another independent optimization for the CEGAR-based approach for
neural network verification was recently proposed
in~\cite{ZhYeGuFuTaJi2022}, where testing methods are embedded during
the formal verification process in order to quickly expose violations
using adversarial attacks.

As far as we know, our work is the first formal verification scheme
which extends the basic mechanism of CEGAR~\cite{ClGrJhLuVe10CEGAR} by
applying abstraction and refinement not only to the checked system but
also to the (output) property. Therefore, it should be compatible with any of
the aforementioned techniques.

\section{Conclusion}
\label{sec:conclusion}
Neural networks are gaining momentum in many areas, but using them
entails various risks. Neural network verification tools seek to help
overcome this issue, but are computationally expensive --- and afford
only limited scalability. Abstraction-based approaches hold great
potential for expediting the verification process and allowing
verifiers to scale to much larger networks. Here, we took a step in
this direction, by applying abstraction and refinement not only to
neural networks themselves but also to the properties being checked
as part of the verification query. We demonstrated that our method
dramatically improves performance, by reducing the number of spurious
counter-examples encountered. The result is a significant boost to the
scalability of existing verification technology.

Moving forward, we plan to pursue several research directions.  First,
we plan to explore additional techniques for property tightening given
an abstract neural network. Second, we intend to develop novel
abstraction and refinement heuristics, which are optimized for
\cegarnntwo{} --- i.e., which will allow better property
tightening. Third, we plan to integrate our method with additional,
recently-proposed CEGAR-based approaches~\cite{AsHaKrMu20, PrAf2020}.

\medskip
\noindent
\textbf{Acknowledgements.}  We thank Orna Kupferman and Orna Grumberg
for their insightful comments.  This project was partially supported
by grants from the National Science Foundation (2211505), the
Binational Science Foundation (2021769), and the Israel Science
Foundation (683/18).

\newpage
\bibliographystyle{abbrv}
\bibliography{property_abstraction}

\begin{thebibliography}{10}

\bibitem{AlSaViGlIf2020}
M.~Alam, M.~Samad, L.~Vidyaratne, A.~Glandon, and K.~Iftekharuddin.
\newblock {Survey on Deep Neural Networks in Speech and Vision Systems}.
\newblock {\em Neurocomputing}, 417:302--321, 2020.

\bibitem{AmCoYeMaHaFaKa23}
G.~Amir, D.~Corsi, R.~Yerushalmi, L.~Marzari, D.~Harel, A.~Farinelli, and
  G.~Katz.
\newblock {Verifying Learning-Based Robotic Navigation Systems}.
\newblock In {\em Proc. 29th Int. Conf. on Tools and Algorithms for the
  Construction and Analysis of Systems (TACAS)}, pages 607--627, 2023.

\bibitem{AmFrKaMaRe23}
G.~Amir, Z.~Freund, G.~Katz, E.~Mandelbaum, and I.~Refaeli.
\newblock {veriFIRE: Verifying an Industrial, Learning-Based Wildfire Detection
  System}.
\newblock In {\em Proc. 25th Int. Symposium on Formal Methods (FM)}, pages
  648--656, 2023.

\bibitem{AmMaZeKaSc23}
G.~Amir, O.~Maayan, T.~Zelazny, G.~Katz, and M.~Schapira.
\newblock {Verifying Generalization in Deep Learning}.
\newblock In {\em Proc. 34th Int. Conf. on Computer Aided Verification (CAV)},
  2023.

\bibitem{AmScKa21}
G.~Amir, M.~Schapira, and G.~Katz.
\newblock {Towards Scalable Verification of Deep Reinforcement Learning}.
\newblock In {\em Proc. 21st Int. Conf. on Formal Methods in Computer-Aided
  Design (FMCAD)}, pages 193--203, 2021.

\bibitem{AmWuBaKa21}
G.~Amir, H.~Wu, C.~Barrett, and G.~Katz.
\newblock {An SMT-Based Approach for Verifying Binarized Neural Networks}.
\newblock In {\em Proc. 27th Int. Conf. on Tools and Algorithms for the
  Construction and Analysis of Systems (TACAS)}, pages 203--222, 2021.

\bibitem{AmZeKaSc22}
G.~Amir, T.~Zelazny, G.~Katz, and M.~Schapira.
\newblock {Verification-Aided Deep Ensemble Selection}.
\newblock In {\em Proc. 22nd Int. Conf. on Formal Methods in Computer-Aided
  Design (FMCAD)}, pages 27--37, 2022.

\bibitem{AnPaDiCh19}
G.~Anderson, S.~Pailoor, I.~Dillig, and S.~Chaudhuri.
\newblock {Optimization and Abstraction: a Synergistic Approach for Analyzing
  Neural Network Robustness}.
\newblock In {\em Proc. 40th ACM SIGPLAN Conf. on Programming Language Design
  and Implementation (PLDI)}, pages 731--744, 2019.

\bibitem{AnLiSa2005}
Z.~Andraus, M.~Liffiton, and K.~Sakallah.
\newblock {CEGAR-Based Formal Hardware Verification: a Case Study}.
\newblock {\em Ann Arbor}, 1001, 2007.

\bibitem{ApAdMeMePa2021}
E.~Apostolidis, E.~Adamantidou, A.~Metsai, V.~Mezaris, and I.~Patras.
\newblock {Video Summarization using Deep Neural Networks: A Survey}, 2021.
\newblock Technical Report. \url{http://arxiv.org/abs/2101.06072}.

\bibitem{AsHaKrMu20}
P.~Ashok, V.~Hashemi, J.~Kretinsky, and S.~M\"{u}hlberger.
\newblock {DeepAbstract: Neural Network Abstraction for Accelerating
  Verification}.
\newblock In {\em Proc. 18th Int. Symposium on Automated Technology for
  Verification and Analysis (ATVA)}, pages 92--107, 2020.

\bibitem{BaLiJo21}
S.~Bak, C.~Liu, and T.~Johnson.
\newblock {The Second International Verification of Neural Networks Competition
  (VNN-COMP 2021): Summary and Results}, 2021.
\newblock Technical Report. \url{http://arxiv.org/abs/2109.00498}.

\bibitem{BaOzSi2019}
G.~Barbastathis, A.~Ozcan, and G.~Situ.
\newblock {On the use of Deep Learning for Computational Imaging}.
\newblock {\em Optica}, 6:921--943, 2019.

\bibitem{BaKa23}
S.~Bassan and G.~Katz.
\newblock {Towards Formal XAI: Formally Approximate Minimal Explanations of
  Neural Networks}.
\newblock In {\em Proc. 29th Int. Conf. on Tools and Algorithms for the
  Construction and Analysis of Systems (TACAS)}, pages 187--207, 2023.

\bibitem{BeLe2016}
D.~Beyer and T.~Lemberger.
\newblock {Symbolic Execution with CEGAR}.
\newblock In {\em Proc. 7th Int. Symposium on Leveraging Applications of Formal
  Methods (ISoLA)}, pages 195--211, 2016.

\bibitem{BoDeDwFiFlGoJaMoMuZhZhZhZi16}
M.~Bojarski, D.~Del~Testa, D.~Dworakowski, B.~Firner, B.~Flepp, P.~Goyal,
  L.~Jackel, M.~Monfort, U.~Muller, J.~Zhang, X.~Zhang, J.~Zhao, and K.~Zieba.
\newblock {End to End Learning for Self-Driving Cars}, 2016.
\newblock Technical Report. \url{http://arxiv.org/abs/1604.07316}.

\bibitem{BuLuTuToKoKu2019}
R.~Bunel, J.~Lu, I.~Turkaslan, P.~Torr, P.~Kohli, and M.~Kumar.
\newblock {Branch and Bound for Piecewise Linear Neural Network Verification},
  2019.
\newblock Technical Report. \url{https://arxiv.org/abs/1909.06588}.

\bibitem{BuTuToKoKu17}
R.~Bunel, I.~Turkaslan, P.~Torr, P.~Kohli, and M.~Kumar.
\newblock {Piecewise Linear Neural Network Verification: A Comparative Study},
  2017.
\newblock Technical Report. \url{https://arxiv.org/abs/1711.00455v1}.

\bibitem{CaWa2016}
N.~Carlini and D.~Wagner.
\newblock {Towards Evaluating the Robustness of Neural Networks}, 2016.
\newblock Technical Report. \url{http://arxiv.org/abs/1608.04644}.

\bibitem{CaKoDaKoKaAmRe22}
M.~Casadio, E.~Komendantskaya, M.~Daggitt, W.~Kokke, G.~Katz, G.~Amir, and
  I.~Refaeli.
\newblock {Neural Network Robustness as a Verification Property: A Principled
  Case Study}.
\newblock In {\em Proc. 34th Int. Conf. on Computer Aided Verification (CAV)},
  pages 219--231, 2022.

\bibitem{ClGrJhLuVe10CEGAR}
E.~Clarke, O.~Grumberg, S.~Jha, Y.~Lu, and H.~Veith.
\newblock {Counterexample-Guided Abstraction Refinement}.
\newblock In {\em Proc. 12th Int. Conf. on Computer Aided Verification (CAV)},
  pages 154--169, 2010.

\bibitem{CoCo1977}
P.~Cousot and R.~Cousot.
\newblock {Abstract Interpretation: A Unified Lattice Model for Static Analysis
  of Programs by Construction or Approximation of Fixpoints}.
\newblock In {\em Proc. 4th ACM SIGACT-SIGPLAN Symposium on Principles of
  Programming Languages (POPL)}, pages 238--252, 1977.

\bibitem{DeChLeTo2019}
J.~Devlin, M.-W. Chang, K.~Lee, and K.~Toutanova.
\newblock {BERT: Pre-training of Deep Bidirectional Transformers for Language
  Understanding}, 2018.
\newblock Technical Report. \url{http://arxiv.org/abs/1810.04805}.

\bibitem{DuChJhSaTi2019}
S.~Dutta, X.~Chen, S.~Jha, S.~Sankaranarayanan, and A.~Tiwari.
\newblock {Sherlock --- A tool for Verification of Neural Network Feedback
  Systems}.
\newblock In {\em Proc. 22nd ACM Int. Conf. on Hybrid Systems: Computation and
  Control (HSCC)}, pages 262--263, 2019.

\bibitem{DuJhSaTi18}
S.~Dutta, S.~Jha, S.~Sanakaranarayanan, and A.~Tiwari.
\newblock {Output Range Analysis for Deep Neural Networks}.
\newblock In {\em Proc. 10th NASA Formal Methods Symposium (NFM)}, pages
  121--138, 2018.

\bibitem{Ehlers2017}
R.~Ehlers.
\newblock {Formal Verification of Piece-Wise Linear Feed-Forward Neural
  Networks}.
\newblock In {\em Proc. 15th Int. Symp. on Automated Technology for
  Verification and Analysis (ATVA)}, pages 269--286, 2017.

\bibitem{ElCoKa2022}
Y.~Y. Elboher, E.~Cohen, and G.~Katz.
\newblock {Neural Network Verification using Residual Reasoning}.
\newblock In {\em Proc. 20th Int. Conf. on Software Engineering and Formal
  Methods (SEFM)}, pages 173--189, 2022.

\bibitem{ElGoKa2020}
Y.~Y. Elboher, J.~Gottschlich, and G.~Katz.
\newblock {An Abstraction-Based Framework for Neural Network Verification}.
\newblock In {\em Proc. 32nd Int. Conf. on Computer Aided Verification (CAV)},
  pages 43--65, 2020.

\bibitem{ElKaKaSc21}
T.~Eliyahu, Y.~Kazak, G.~Katz, and M.~Schapira.
\newblock {Verifying Learning-Augmented Systems}.
\newblock In {\em Proc. Conf. of the ACM Special Interest Group on Data
  Communication on the Applications, Technologies, Architectures, and Protocols
  for Computer Communication (SIGCOMM)}, pages 305--318, 2021.

\bibitem{GeMiDrTsChVe18}
T.~Gehr, M.~Mirman, D.~Drachsler-Cohen, E.~Tsankov, S.~Chaudhuri, and
  M.~Vechev.
\newblock {AI2: Safety and Robustness Certification of Neural Networks with
  Abstract Interpretation}.
\newblock In {\em Proc. 39th IEEE Symposium on Security and Privacy (S\&P)},
  2018.

\bibitem{GoFeMaBaKa20}
S.~Gokulanathan, A.~Feldsher, A.~Malca, C.~Barrett, and G.~Katz.
\newblock {Simplifying Neural Networks using Formal Verification}.
\newblock In {\em Proc. 12th NASA Formal Methods Symposium (NFM)}, pages
  85--93, 2020.

\bibitem{GoAdKeKa20}
B.~Goldberger, Y.~Adi, J.~Keshet, and G.~Katz.
\newblock {Minimal Modifications of Deep Neural Networks using Verification}.
\newblock In {\em Proc. 23rd Int. Conf. on Logic for Programming, Artificial
  Intelligence and Reasoning (LPAR)}, 2020.

\bibitem{Goodfellow-et-al-2016}
I.~Goodfellow, Y.~Bengio, and A.~Courville.
\newblock {\em Deep Learning}.
\newblock MIT Press, 2016.
\newblock \url{http://www.deeplearningbook.org}.

\bibitem{GoDvStBuQiUeArMaKo2018}
S.~Gowal, K.~Dvijotham, R.~Stanforth, R.~Bunel, C.~Qin, J.~Uesato,
  R.~Arandjelovic, T.~Mann, and P.~Kohli.
\newblock {On the Effectiveness of Interval Bound Propagation for Training
  Verifiably Robust Models}, 2018.
\newblock Technical Report. \url{http://arxiv.org/abs/1810.12715}.

\bibitem{GuQiChPaZhYuHaWaZhWuPa2020}
A.~Gulati, J.~Qin, C.-C. Chiu, N.~Parmar, Y.~Zhang, J.~Yu, W.~Han, S.~Wang,
  Z.~Zhang, Y.~Wu, and R.~Pang.
\newblock {Conformer: Convolution-Augmented Transformer for Speech
  Recognition}, 2020.
\newblock Technical Report. \url{http://arxiv.org/abs/2005.08100}.

\bibitem{GuLi2020}
J.~Guo and C.~Liu.
\newblock {Practical Poisoning Attacks on Neural Networks}.
\newblock In {\em Proc. 22nd European Conf. on Computer Vision (ECCV)}, pages
  142--158, 2020.

\bibitem{heZaReSu2015deep}
K.~He, X.~Zhang, S.~Ren, and J.~Sun.
\newblock {Deep Residual Learning for Image Recognition}.
\newblock In {\em Proc. IEEE Conf. on Computer Vision and Pattern Recognition
  (CVPR)}, pages 770--778, 2016.

\bibitem{HuKwWaWu17}
X.~Huang, M.~Kwiatkowska, S.~Wang, and M.~Wu.
\newblock {Safety Verification of Deep Neural Networks}.
\newblock In {\em Proc. 29th Int. Conf. on Computer Aided Verification (CAV)},
  pages 3--29, 2017.

\bibitem{IsBaZhKa22}
O.~Isac, C.~Barrett, M.~Zhang, and G.~Katz.
\newblock {Neural Network Verification with Proof Production}.
\newblock In {\em Proc. 22nd Int. Conf. on Formal Methods in Computer-Aided
  Design (FMCAD)}, pages 38--48, 2022.

\bibitem{IsZoBa23}
O.~Isac, Y.~Zohar, C.~Barrett, and G.~Katz.
\newblock {DNN Verification, Reachability, and the Exponential Function
  Problem}, 2023.
\newblock Technical Report. \url{https://arxiv.org/abs/2305.06064}.

\bibitem{JaBaKa20}
Y.~Jacoby, C.~Barrett, and G.~Katz.
\newblock {Verifying Recurrent Neural Networks using Invariant Inference}.
\newblock In {\em Proc. 18th Int. Symposium on Automated Technology for
  Verification and Analysis (ATVA)}, pages 57--74, 2020.

\bibitem{JaKlMaCl2016}
M.~Janota, W.~Klieber, J.~Marques-Silva, and E.~Clarke.
\newblock {Solving QBF with Counterexample Guided Refinement}.
\newblock {\em Artificial Intelligence}, 234:1--25, 2016.

\bibitem{JoKh2019}
J.~Johnson and T.~Khoshgoftaar.
\newblock {Survey on Deep Learning with Class Imbalance}.
\newblock {\em Journal of Big Data}, 6(1):1--54, 2019.

\bibitem{JuLoBrOwKo2016}
K.~Julian, J.~Lopez, J.~Brush, M.~Owen, and M.~Kochenderfer.
\newblock {Policy Compression for Aircraft Collision Avoidance Systems}.
\newblock In {\em Proc. 35th Digital Avionics Systems Conf. (DASC)}, pages
  1--10, 2016.

\bibitem{KaToShLeSuFe2014}
A.~Karpathy, G.~Toderici, S.~Shetty, T.~Leung, R.~Sukthankar, and L.~Fei-Fei.
\newblock {Large-Scale Video Classification with Convolutional Neural
  Networks}.
\newblock In {\em Proc. IEEE Conf. on Computer Vision and Pattern Recognition
  (CVPR)}, pages 1725--1732, 2014.

\bibitem{KaBaDiJuKo17Reluplex}
G.~Katz, C.~Barrett, D.~Dill, K.~Julian, and M.~Kochenderfer.
\newblock {Reluplex: An Efficient SMT Solver for Verifying Deep Neural
  Networks}.
\newblock In {\em Proc. 29th Int. Conf. on Computer Aided Verification (CAV)},
  pages 97--117, 2017.

\bibitem{KaBaDiJuKo21}
G.~Katz, C.~Barrett, D.~Dill, K.~Julian, and M.~Kochenderfer.
\newblock {Reluplex: a Calculus for Reasoning about Deep Neural Networks}.
\newblock {\em Formal Methods in System Design (FMSD)}, 2021.

\bibitem{KaHuIbJuLaLiShThWuZeDiKoBa19Marabou}
G.~Katz, D.~Huang, D.~Ibeling, K.~Julian, C.~Lazarus, R.~Lim, P.~Shah,
  S.~Thakoor, H.~Wu, A.~Zelji\'c, D.~Dill, M.~Kochenderfer, and C.~Barrett.
\newblock {The Marabou Framework for Verification and Analysis of Deep Neural
  Networks}.
\newblock In {\em Proc. 31st Int. Conf. on Computer Aided Verification (CAV)},
  pages 443--452, 2019.

\bibitem{KaBaKaSc19}
Y.~Kazak, C.~Barrett, G.~Katz, and M.~Schapira.
\newblock {Verifying Deep-RL-Driven Systems}.
\newblock In {\em Proc. 1st ACM SIGCOMM Workshop on Network Meets AI \& ML
  (NetAI)}, 2019.

\bibitem{WoKo2017}
Z.~Kolter and E.~Wong.
\newblock {Provable Defenses Against Adversarial Examples via the Convex Outer
  Adversarial Polytope}.
\newblock In {\em Proc. 16th IEEE Int. Conf. on Machine Learning and
  Applications (ICML)}, 2018.

\bibitem{KrDvStGoMaKo2018}
D.~Krishnamurthy, R.~Stanforth, S.~Gowal, T.~Mann, and P.~Kohli.
\newblock {A Dual Approach to Scalable Verification of Deep Networks}, 2018.
\newblock Technical Report. \url{http://arxiv.org/abs/1803.06567}.

\bibitem{KuKaGoJuBaKo18}
L.~Kuper, G.~Katz, J.~Gottschlich, K.~Julian, C.~Barrett, and M.~Kochenderfer.
\newblock {Toward Scalable Verification for Safety-Critical Deep Networks},
  2018.
\newblock Technical Report. \url{https://arxiv.org/abs/1801.05950}.

\bibitem{LaKa21}
O.~Lahav and G.~Katz.
\newblock {Pruning and Slicing Neural Networks using Formal Verification}.
\newblock In {\em Proc. 21st Int. Conf. on Formal Methods in Computer-Aided
  Design (FMCAD)}, pages 183--192, 2021.

\bibitem{Le98}
Y.~LeCun.
\newblock {The MNIST Database of Handwritten Digits}, 1998.
\newblock \url{http://yann.lecun.com/exdb/mnist/}.

\bibitem{LeMcCo2021}
T.~Lee, S.~Mckeever, and J.~Courtney.
\newblock {Flying Free: A Research Overview of Deep Learning in Drone
  Navigation Autonomy}.
\newblock {\em Drones}, 5(2), 2021.

\bibitem{LiArLaBaKo19}
C.~Liu, T.~Arnon, C.~Lazarus, C.~Barrett, and M.~Kochenderfer.
\newblock {Algorithms for Verifying Deep Neural Networks}, 2019.
\newblock Technical Report. \url{http://arxiv.org/abs/1903.06758}.

\bibitem{LiXiShSoXuMi2022}
J.~Liu, Y.~Xing, X.~Shi, F.~Song, Z.~Xu, and Z.~Ming.
\newblock {Abstraction and Refinement: Towards Scalable and Exact Verification
  of Neural Networks}, 2022.
\newblock Technical Report. \url{https://arxiv.org/abs/2207.00759}.

\bibitem{Liu2017}
W.~Liu, Z.~Wang, X.~Liu, N.~Zeng, Y.~Liu, and F.~Alsaadi.
\newblock {A Survey of Deep Neural Network Architectures and their
  Applications}.
\newblock {\em Neurocomputing}, 234:11--26, 2017.

\bibitem{LoMa17}
A.~Lomuscio and L.~Maganti.
\newblock {An Approach to Reachability Analysis for Feed-Forward ReLU Neural
  Networks}, 2017.
\newblock Technical Report. \url{https://arxiv.org/abs/1706.07351}.

\bibitem{MaAbNeCh2017}
T.~Malekzadeh, M.~Abdollahzadeh, H.~Nejati, and N.-M. Cheung.
\newblock {Aircraft Fuselage Defect Detection using Deep Neural Networks},
  2017.
\newblock Technical Report. \url{http://arxiv.org/abs/1712.09213}.

\bibitem{Ma21}
V.~Mazzia, F.~Salvetti, and M.~Chiaberge.
\newblock Efficient-{CapsNet}: capsule network with self-attention routing.
\newblock {\em Scientific Reports}, 11(1), jul 2021.

\bibitem{MuMaSiPuVe2022}
M.~M\"uller, G.~Makarchuk, G.~Singh, M.~P\"uschel, and M.~Vechev.
\newblock {PRIMA: General and Precise Neural Network Certification via Scalable
  Convex Hull Approximations}.
\newblock In {\em Proc. 49th ACM SIGPLAN Symposium on Principles of Programming
  Languages (POPL)}, 2022.

\bibitem{NaHi10}
V.~Nair and G.~Hinton.
\newblock {Rectified Linear Units Improve Restricted Boltzmann Machines}.
\newblock In {\em Proc. 27th Int. Conf. on Machine Learning (ICML)}, pages
  807--814, 2010.

\bibitem{NaKaRySaWa17}
N.~Narodytska, S.~Kasiviswanathan, L.~Ryzhyk, M.~Sagiv, and T.~Walsh.
\newblock {Verifying Properties of Binarized Deep Neural Networks}, 2017.
\newblock Technical Report. \url{http://arxiv.org/abs/1709.06662}.

\bibitem{JoErBe2015}
J.~Nellen, E.~{\'A}brah{\'a}m, and B.~Wolters.
\newblock {A CEGAR Tool for the Reachability Analysis of PLC-Controlled Plants
  Using Hybrid Automata}.
\newblock In {\em Proc. 3rd IEEE Int. Workshop on Formal Methods Integration
  (FMi)}, pages 55--78, 2015.

\bibitem{OsBaKa22}
M.~Ostrovsky, C.~Barrett, and G.~Katz.
\newblock {An Abstraction-Refinement Approach to Verifying Convolutional Neural
  Networks}.
\newblock In {\em Proc. 20th. Int. Symposium on Automated Technology for
  Verification and Analysis (ATVA)}, 2022.

\bibitem{PaWuGrCaPaBa2021}
C.~Paterson, H.~Wu, J.~Grese, R.~Calinescu, C.~Pasareanu, and C.~Barrett.
\newblock {DeepCert: Verification of Contextually Relevant Robustness for
  Neural Network Image Classifiers}.
\newblock In {\em Proc. 40th Int. Comf. on Computer Safety, Reliability, and
  Security (SAFECOMP)}, pages 3--17, 2021.

\bibitem{PeCaYaJa2019}
K.~Pei, Y.~Cao, J.~Yang, and S.~Jana.
\newblock {DeepXplore: Automated Whitebox Testing of Deep Learning Systems}.
\newblock {\em Communications of the ACM (CACM)}, pages 137--145, 2019.

\bibitem{PrAf2020}
P.~Prabhakar and Z.~Afzal.
\newblock {Abstraction based Output Range Analysis for Neural Networks}, 2020.
\newblock Technical Report. \url{http://arxiv.org/abs/2007.09527}.

\bibitem{PuTa10}
L.~Pulina and A.~Tacchella.
\newblock {An Abstraction-Refinement Approach to Verification of Artificial
  Neural Networks}.
\newblock In {\em Proc. 22nd Int. Conf. on Computer Aided Verification (CAV)},
  pages 243--257, 2010.

\bibitem{ReKa22}
I.~Refaeli and G.~Katz.
\newblock {Minimal Multi-Layer Modifications of Deep Neural Networks}.
\newblock In {\em Proc. 5th Workshop on Formal Methods for ML-Enabled
  Autonomous Systems (FoMLAS)}, 2022.

\bibitem{ReZhQiLi2020}
K.~Ren, T.~Zheng, Z.~Qin, and X.~Liu.
\newblock {Adversarial Attacks and Defenses in Deep Learning}.
\newblock {\em Engineering}, 6(3):346--360, 2020.

\bibitem{HeRuGo2017}
H.~Riener, R.~Ehlers, and G.~Fey.
\newblock {CEGAR-Based EF Synthesis of Boolean Functions with an Application to
  Circuit Rectification}.
\newblock In {\em Proc. 22nd Asia and South Pacific Design Automation
  Conference (ASP-DAC)}, pages 251--256, 2017.

\bibitem{ShFaDeChAl2021}
S.~Shamshirband, M.~Fathi, A.~Dehzangi, A.~Chronopoulos, and H.~Alinejad-Rokny.
\newblock {A Review on Deep Learning Approaches in Healthcare Systems:
  Taxonomies, Challenges, and Open Issues}.
\newblock {\em Journal of Biomedical Informatics}, 113, 2021.

\bibitem{SiGePuVe2019}
G.~Singh, T.~Gehr, M.~Puschel, and M.~Vechev.
\newblock {An Abstract Domain for Certifying Neural Networks}.
\newblock In {\em Proc. 46th ACM SIGPLAN Symposium on Principles of Programming
  Languages (POPL)}, 2019.

\bibitem{SoTh2020}
M.~Sotoudeh and A.~Thakur.
\newblock {Abstract Neural Networks}, 2020.
\newblock Technical Report. \url{http://arxiv.org/abs/2009.05660}.

\bibitem{SrHiKrSuSa2014}
N.~Srivastava, G.~Hinton, A.~Krizhevsky, I.~Sutskever, and R.~Salakhutdinov.
\newblock {Dropout: A Simple Way to Prevent Neural Networks from Overfitting}.
\newblock {\em Journal of Machine Learning Research}, 15(56):1929--1958, 2014.

\bibitem{StWuZeJuKaBaKo21}
C.~Strong, H.~Wu, A.~Zelji\'c, K.~Julian, G.~Katz, C.~Barrett, and
  M.~Kochenderfer.
\newblock {Global Optimization of Objective Functions Represented by ReLU
  Networks}.
\newblock {\em Journal of Machine Learning}, pages 1--28, 2021.

\bibitem{SuMiLiJi2021}
H.~Su, M.~Chai, L.~Chen, and J.~Lv.
\newblock {Deep Learning-Based Model Predictive Control for Virtual Coupling
  Railways Operation}.
\newblock In {\em Proc. 24th IEEE Int. Intelligent Transportation Systems Conf
  (ITSC)}, pages 3490--3495, 2021.

\bibitem{TiPeJaRa2017}
Y.~Tian, K.~Pei, S.~Jana, and R.~Baishakhi.
\newblock {DeepTest: Automated Testing of Deep-Neural-Network-driven Autonomous
  Cars}.
\newblock In {\em Proc. 40th Int. Conf. on Software Engineering (ICSE)}, pages
  303--314, 2018.

\bibitem{TjXiTe19}
V.~Tjeng, K.~Xiao, and R.~Tedrake.
\newblock {Evaluating Robustness of Neural Networks with Mixed Integer
  Programming}.
\newblock In {\em Proc. 7th Int. Conf. on Learning Representations (ICLR)},
  2019.

\bibitem{WaPeWhYaJa2018}
S.~Wang, K.~Pei, J.~Whitehouse, J.~Yang, and S.~Jana.
\newblock {Formal Security Analysis of Neural Networks using Symbolic
  Intervals}.
\newblock In {\em Proc. 27th USENIX Security Symposium}, 2018.

\bibitem{WaZhXuLiJaHsKo2021}
S.~Wang, H.~Zhang, K.~Xu, X.~Lin, S.~Jana, C.-J. Hsieh, and Z.~Kolter.
\newblock {Beta-CROWN: Efficient Bound Propagation with Per-Neuron Split
  Constraints for Complete and Incomplete Neural Network Verification}.
\newblock In {\em Proc. 35th Conf. on Neural Information Processing Systems
  (NeurIPS)}, 2021.

\bibitem{WuOzZeIrJuGoFoKaPaBa20}
H.~Wu, A.~Ozdemir, A.~Zelji\'c, A.~Irfan, K.~Julian, D.~Gopinath, S.~Fouladi,
  G.~Katz, C.~P\u{a}s\u{a}reanu, and C.~Barrett.
\newblock {Parallelization Techniques for Verifying Neural Networks}.
\newblock In {\em Proc. 20th Int. Conf. on Formal Methods in Computer-Aided
  Design (FMCAD)}, pages 128--137, 2020.

\bibitem{WuZeKaBa22}
H.~Wu, A.~Zelji\'c, G.~Katz, and C.~Barrett.
\newblock {Efficient Neural Network Analysis with Sum-of-Infeasibilities}.
\newblock In {\em Proc. 28th Int. Conf. on Tools and Algorithms for the
  Construction and Analysis of Systems (TACAS)}, pages 143--163, 2022.

\bibitem{XiTrJo18}
W.~Xiang, H.-D. Tran, and T.~Johnson.
\newblock {Output Reachable Set Estimation and Verification for Multilayer
  Neural Networks}.
\newblock {\em IEEE Transactions on Neural Networks and Learning Systems
  (TNNLS)}, 99:1--7, 2018.

\bibitem{HaYaHaDeHuJiAn2019}
H.~Xu, Y.~Ma, H.~Liu, D.~Deb, H.~Liu, J.~Tang, and A.~Jain.
\newblock {Adversarial Attacks and Defenses in Images, Graphs and Text: A
  Review}, 2019.
\newblock Technical Report. \url{http://arxiv.org/abs/1909.08072}.

\bibitem{KaZhHuYiKaMiBhXuCh2020}
K.~Xu, Z.~Shi, H.~Zhang, Y.~Wang, K.-W. Chang, M.~Huang, B.~Kailkhura, X.~Lin,
  and C.-J. Hsieh.
\newblock {Automatic Perturbation Analysis for Scalable Certified Robustness
  and Beyond}.
\newblock In {\em Proc. Advances in Neural Information Processing Systems},
  pages 1129--1141, 2020.

\bibitem{YaLiLiHuWaSuXuZh2020}
P.~Yang, R.~Li, J.~Li, C.-C. Huang, J.~Wang, J.~Sun, B.~Xue, and L.~Zhang.
\newblock {Improving Neural Network Verification through Spurious Region Guided
  Refinement}, 2020.
\newblock Technical Report. \url{http://arxiv.org/abs/2010.07722}.

\bibitem{ZeWuBaKa22}
T.~Zelazny, H.~Wu, C.~Barrett, and G.~Katz.
\newblock {On Reducing Over-Approximation Errors for Neural Network
  Verification}.
\newblock In {\em Proc. 22nd Int. Conf. on Formal Methods in Computer-Aided
  Design (FMCAD)}, pages 17--26, 2022.

\bibitem{ZhViMuBe2018}
C.~Zhang, O.~Vinyals, R.~Munos, and S.~Bengio.
\newblock {A Study on Overfitting in Deep Reinforcement Learning}, 2018.
\newblock Technical Report. \url{http://arxiv.org/abs/1804.06893}.

\bibitem{ZhWeChHsDa2018}
H.~Zhang, T.-W. Weng, P.-Y. Chen, C.-J. Hsieh, and L.~Daniel.
\newblock {Efficient Neural Network Robustness Certification with General
  Activation Functions}, 2018.
\newblock Technical Report. \url{http://arxiv.org/abs/1811.00866}.

\bibitem{ZhYeGuFuTaJi2022}
Z.~Zhao, Y.~Zhang, G.~Chen, F.~Song, T.~Chen, and J.~Liu.
\newblock {CLEVEREST: Accelerating CEGAR-based Neural Network Verification via
  Adversarial Attacks}.
\newblock In {\em Proc. 29th Static Analysis Symposium (SAS)}, 2022.

\end{thebibliography}

\end{document}